\documentclass[journal]{IEEEtran}

\usepackage{amsmath,graphicx}
\usepackage{amsfonts}
\usepackage{amssymb,amsmath}
\usepackage{amsthm}

\usepackage{subfig}
\usepackage{algorithm}
\usepackage{algpseudocode}
\usepackage{pifont}
\usepackage{cite}
\usepackage{footnote}
\makesavenoteenv{tabular}

\newtheorem{theorem}{Theorem}[section]

\newtheorem{lemma}[theorem]{Lemma}

\begin{document}
\title{Autonomous Localization and Mapping Using a Single Mobile Device}
%
\author{Tiexing Wang, Fangrong Peng and Biao Chen}


%
\maketitle
\begin{abstract}
This paper considers the problem of simultaneous $2$-$D$ room shape reconstruction and self-localization without the requirement of any pre-established infrastructure. A mobile device equipped with co-located microphone and loudspeaker as well as internal motion sensors is used to emit acoustic pulses and collect echoes reflected by the walls. Using only first order echoes, room shape recovery and self-localization is feasible when auxiliary information is obtained using motion sensors. In particular, it is established that using echoes collected at three measurement locations and the two distances between consecutive measurement points, unique localization and mapping can be achieved provided that the three measurement points are not collinear. Practical algorithms for room shape reconstruction and self-localization in the presence of noise and higher order echoes are proposed along with experimental results to demonstrate the effectiveness of the proposed approach.
\end{abstract}


\section{introduction}
Indoor localization has become more important in recent years as numerous applications, e.g., public safety or location based services, rely on accurate indoor localization \cite{Liu}. As GPS signals are severely attenuated in typical indoor environment, a number of alternative technologies have been proposed for indoor localization, e.g. those using WiFi \cite{Chintalapudi,Lim,Martin}, UWB signal \cite{Stelios,Schroeder,Zhou}, LED light \cite{Jung,Vegni} 
, or some combination of the above.\par
These technologies inevitably require indoor geometry information. There are applications where the indoor room geometry may need to be acquired concurrently with localization. This is generally referred to as simultaneous localization and mapping (SLAM). We comment that the so-called WiFi-SLAM still requires indoor mapping information; SLAM refers to the training process that associates mapping information with the WiFi signature \cite{Huang}. There are also applications where mapping itself is the ultimate goal instead of self-localization \cite{Canclini,Kuang}.\par
For many applications where room shape reconstruction is required, acoustic based approach is arguably more suitable as rooms are often defined by dominant sound reflectors (walls). The distance measurements as measured through acoustic echoes contain rich information about the location of the measurement points as well as the room geometry. 
A key advantage of the acoustic based approach is that no pre-established infrastructure is needed; this is in sharp contrast with other approaches which inevitably require either deployment of anchor nodes \cite{Blanco,Djugash} or the availability of ambient WiFi signals as well as preliminary maps \cite{Huang}. This unique advantage has the potential to broaden the applications of indoor mapping and localization to systems where current technologies are either unsuitable or too expensive to implement.\par
The most prevalent acoustic based approach is to employ a single fixed loudspeaker and a microphone array, or equivalently, a fixed loudspeaker and a mobile microphone \cite{Dokmanic1,Tervo2,Antonacci,Dokmanic2,Ba,Crocco}. It was shown that both the room shape and the geometry of the microphone array (or the trajectory of the mobile microphone) can be estimated by first order echoes \cite{Dokmanic3}. Furthermore, bearing only SLAM can be achieved using a mobile microphone array \cite{Evers}.\par
The fact that a microphone array needs to be deployed leaves much to be desired: fully autonomous SLAM should require minimum deployment effort. Ideally, a single mobile device that moves around would autonomously reconstruct the room shape while tracking its own movement within the recovered room geometry. Indeed, room shape recovery using a single acoustic device has been addressed in the literature. It was established that any convex polygon can be reconstructed by the \emph{entire} set of both first and second order echoes collected using a fixed device with a collocated microphone and loudspeaker \cite{Dokmanic2}. However, experimental results, including that of our own, demonstrated that higher order acoustic echoes are often difficult to recover, thus the requirement of having the entire set of second order echoes makes such an approach impractical.\par
On the other hand, given only \emph{grouped} first order echoes, SLAM can be achieved for a large class of convex polygon other than parallelograms \cite{Fangrong}. This result was strengthened in \cite{Krekovic1} where it was established that parallelograms are the only convex polygons that are not recoverable via grouped first order echoes. Here ``-\emph{grouped}'' means correct labeling, i.e., the correspondence between collected echoes and walls is known.\par
This paper makes further progress in overcoming the shortcomings of the approaches in \cite{Fangrong,Krekovic1}. The reconstruction will again be based on first order echoes only but without the knowledge of echo labeling. To overcome the ambiguity associated with parallelograms, our approach leverages the ever expanding capability of various motion sensors embedded in latest smart phones, including accelerometer, magnetometer, and gyroscope. Those sensors are capable of measuring distance and direction information of a moving device \cite{Kang,Li,Roy}. However, existing results indicate that while distance measures have reasonable accuracy, direction measurement is often subject to large measurement error \cite{Zhang}. Thus our current approach only exploits the distance measurements and the key question to be addressed is how much additional information will be needed for acoustic SLAM to be able to recover all convex polygons.\par
The major contribution of the paper is to establish that with three non-collinear measurement points, SLAM can be achieved for all convex polygons using \emph{ungrouped} first order echoes provided that the distances between consecutive measurement points are known. Note that this additional information is much weaker than the knowledge of the complete geometry of the measurement - this is tantamount to knowing only two sides of a triangle which is inadequate to construct the triangle. An added advantage of this additional distance information is that it removes the need for grouped echoes, making the scheme much more widely applicable as it can accommodate a great deal of freedom in the movement of the device. Preliminary results have been reported in \cite{Tiexing}. The present work, in addition to expanding on technical details, contains several new results including a more detailed analysis on exactly what is the minimum amount of distance information that is needed for SLAM. Specifically, it is further established that with ungrouped echoes, a single distance measure does not suffice for parallelograms. Note the subtle but important difference with that of \cite{Fangrong,Krekovic1} in which grouped instead of ungrouped echoes are assumed.\par
The rest of the paper is organized as follows. Section II introduces the indoor propagation model of acoustic signals, image source model and existing results on $2$-D with a single device. Theoretical guarantee of successful SLAM given distances between consecutive measurement points is provided in Section III along with a practical algorithm that handles the presence of measurement noise and higher order/spurious peaks. Experiment results are provided in Section IV followed by conclusion in Section V.
\section{problem statement}
\subsection{Room Impulse Response Model}
Acoustic signal propagation from a loudspeaker to a microphone in a room can be described by the room impulse response (RIR), which includes both line-of-sight (LOS) and reflected components. If the microphone and loudspeaker are much closer to each other compared to the distance between the device and the walls, we say it is a co-located device. For a co-located device at the $j$th measurement point denoted by $O_{j}$, the RIR is, ignoring dispersion,
\begin{equation}\label{sys1}
h^{(j)}(t)=\sum_{i}\alpha_{i}^{(j)}\delta(t-\tau_{i}^{(j)}),
\end{equation}
where $\alpha_{i}^{(j)}$'s and $\tau_{i}^{(j)}$'s are path gains and delays from the transmitter to the receiver, respectively. Since higher order reflective paths typically have much weaker power, $h^{(j)}(t)$ can be approximated by the first $N_{j}+1$ components including LOS and $N_{j}$ reflective paths:
\begin{equation*}
h^{(j)}(t)\approx\sum_{i=0}^{N_{j}}\alpha_{i}^{(j)}\delta(t-\tau_{i}^{(j)}),
\end{equation*}
where we assume that the $N_{j}$ reflective paths contain all first order reflections and higher order ones that are detectable. Notice that for an arbitrary convex polygon, not every measurement point has first order echoes to all the walls. We refer to those measurement points can receive all first order echoes as \emph{feasible} measurement points.\par
Denote by $s(t)$ the emitted signal at the speaker. Then the received signal at the microphone for the $j$th measurement point is
\begin{equation}\label{received_signal}
r^{(j)}(t) = s(t) * h^{(j)}(t)+\omega(t),
\end{equation}
where $*$ denotes linear convolution and $\omega(t)$ is the additive noise. Ideally, the delays can be recovered from the received signal $r^{(j)}(t)$ if $s(t)$ behaves like a Dirac delta function \cite{Antonacci}. However, this requires a wideband acoustic signal along with a wideband acoustic channel, including that of the microphone receiver. A more practical alternative is to emit $s(t)$ with a desired auto-correlation function that is \emph{peaky} and then implement a correlator at the microphone:
\begin{equation}\label{correlation1}
m^{(j)}(t) =  r^{(j)}(t) * s(t).
\end{equation}
Thus, the first and dominant peak of $m^{(j)}(t)$ corresponds to the LOS components, while the remaining peaks correspond to reflective components. The time difference of arrival (TDOA) in reference to the LOS component can be used for estimating the delays of different reflective paths. A simple peak-detection method will be introduced in Section V.A, where the chirp signal is used for $s(t)$ because of its nice auto-correlation property.\par
Define a column vector
\begin{equation}\label{candd}
\tilde{\mathbf{r}}_{j}=\bigg\{\frac{(\tau_{i}^{(j)}-\tau_{0}^{(j)})c}{2}\bigg\}_{i=1}^{N_{j}},
\end{equation}
where $c$ is the speed of sound and $\tau_{i}^{(j)}$ is the arrival time of the $i$th path with $\tau_{0}^{(j)}$ corresponding to the LOS component. Then $\tilde{\mathbf{r}}_{j}$ contains all the distances between the device and the walls, along with some higher order terms.\par
\subsection{Image Source Model}
With the image source model \cite{Dokmanic1}, reflections within a constrained space can be viewed as free space LOS propagations from virtual sources to the receiver. Let the coordinate of $O_{j}$ be denoted by $\mathbf{o}_{j}$. As show in Fig.~$1$, the first order image source of $O_{j}$ with respect to the $i$th wall is
\begin{figure}[btp]
    \center
    \includegraphics[height=50mm,width=50mm]{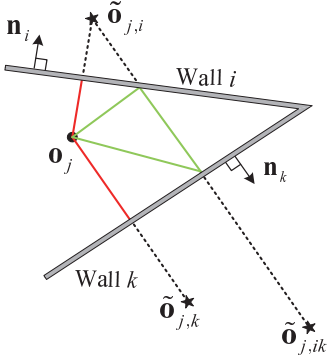}
    \caption{The image source model: $\tilde{\mathbf{o}}_{j,i}$ and $\tilde{\mathbf{o}}_{j,k}$ are first-order image sources with respect to the $i$th and $k$th wall and $\tilde{\mathbf{o}}_{j,ik}$ is the second-order image source with respect to the $i$th and $k$th wall in the stated order.}
    \label{Fig1}
\end{figure}
\begin{equation*}
\tilde{\mathbf{o}}_{j,i} = 2\langle \mathbf{p}_{i}-\mathbf{o}_{j},\mathbf{n}_{i}\rangle\mathbf{n}_{i} + \mathbf{o}_{j},
\end{equation*}
where $\mathbf{p}_{i}$ is any point on the $i$th wall, $\mathbf{n}_{i}$ is the outward norm vector of the $i$th wall and $\langle\mathbf{x},\mathbf{y}\rangle$ denotes the inner product between $\mathbf{x}$ and $\mathbf{y}$.
Let $r_{j,i}$ be the distance between $O_{j}$ and the $i$th wall, then
\begin{equation}\label{distance1}
r_{j,i}=\frac{1}{2}||\tilde{\mathbf{o}}_{j,i} - \mathbf{o}_{j}||_{2}.
\end{equation}
Moreover, the second order image source of $O_{j}$ with respect to the $i$th and the $k$th wall is
\begin{equation*}
\tilde{\mathbf{o}}_{j,ik}=2\langle\mathbf{p}_{k}-\tilde{\mathbf{o}}_{j,i},\mathbf{n}_{k}\rangle\mathbf{n}_{k}+\tilde{\mathbf{o}}_{j,i}.
\end{equation*}
Similarly, we denote by $r_{j,ik}$ the half distance between $\mathbf{o}_{j}$ and $\tilde{\mathbf{o}}_{j,ik}$. Following similar steps, higher order image sources can be represented by lower order image sources. Then all the elements of $\tilde{\mathbf{r}}_{j}$ can be represented by the real source and image sources. For the rest of the paper, the term \emph{echo} is used to refer to either the delay $\tau_{i}^{(j)}$ or the corresponding elements of $\tilde{\mathbf{r}}_{j}$ if no ambiguity occurs.
\subsection{Two Extreme Cases}
The most benign case is when the location of the measurement points are known, or equivalently, the distance between pairwise measurement points are given \cite{Krekovic}. In this case, only room shape reconstruction is of interest and the problem becomes trivial, at least in the noiseless case. It amounts to finding common tangent lines of circles centered at three non-collinear measurement points.\par
The other extreme is when the reconstruction is free of any geometry information of the measurement points. In this case, both room shape and self-localization are of interest. This was first investigated in \cite{Fangrong} where it was established that a large class of convex polygons can be reconstructed by \emph{grouped} first order echoes and, subsequently, the coordinates of measurement points can be also estimated. An important exception is parallelograms and it was shown in \cite{Fangrong} that unique reconstruction of parallelograms is impossible using first-order echoes alone. The result was later strengthened in \cite{Krekovic1} where it was proved that all convex polygons except parallelogram can be reconstructed subject to the usual rotation and reflection ambiguities.\par
\section{SLAM with known path lengths}
\subsection{SLAM with Two Path Lengths}
\begin{figure}[btp]
    \center
    \includegraphics[height=48mm,width=72mm]{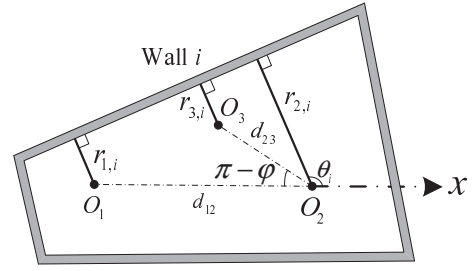}
    \caption{A mobile device is employed to achieve SLAM. The mobile device emits signal and collects echoes at $O_{1}$, $O_{2}$ and $O_{3}$ successively. The distances between the consecutive measurement points are $d_{12}$ and $d_{23}$.}
    \label{Fig2}
\end{figure}
Consider a convex planar $K$-polygon. As shown in Fig.$2$, a mobile device with co-located microphone and loudspeaker emits pulses and receives echoes at $\{O_{j}\}_{j=1}^{3}$. Without loss of generality, we assume that $O_{1}$ is the origin, $O_{2}$ lies on the $x$-axis, and $O_{3}$ lies above the $x$-axis. Let $\varphi=(\pi-\angle O_{1}O_{2}O_{3})\in(0,\pi)$ and the lengths of $O_{1}O_{2}$ and $O_{2}O_{3}$ be denoted by $d_{12}$ and $d_{23}$, respectively.\footnote{If $\varphi\in (0,2\pi)$, i.e. we do not have control of where to place $O_{3}$, then the reconstruction is subject to reflection ambiguity (c.f. Theorem III.3).} Suppose the mobile device is capable of measuring its path length when moving from one place to another, i.e. $d_{12}$ and $d_{23}$ are known. Our goal is to simultaneously determine the room shape and the coordinate of $O_{3}$ using first-order echoes.\par
From Fig.~$2$, it is straightforward to show that
\begin{equation}\label{rel1}
(r_{2,i}-r_{1,i})+d_{12}\cos\theta_{i}=0,
\end{equation}
\begin{equation}\label{d23-1}
d_{23}\cos(\theta_{i}-\varphi)+(r_{3,i}-r_{2,i})=0.
\end{equation}
\subsection{Ideal Case}
Let $\mathbf{r}_{j}=\{r_{j,i}\}_{i=1}^{K}$ be a column vector with its entries defined in \eqref{distance1}. We assume for now that, the one-to-one mapping $f_{j}:\tilde{\mathbf{r}}_{j}\mapsto \mathbf{r}_{j}$ is known for all $j$'s. In other words, $r_{j,i}$'s have been correctly chosen from $\tilde{\mathbf{r}}_{j}$ for $j=1,2,3$ and $i=1,\ldots,K$. For the rest of the paper, we say that the received echoes are \emph{grouped} if echoes are correctly labeled. The remaining problem is to determine the uniqueness of $\theta_{i}$'s and $\varphi$ given \eqref{rel1} and \eqref{d23-1}.\par
Define $\alpha_{ii^{\prime}}=-\frac{r_{2,i}-r_{1,i^{\prime}}}{d_{12}}$ and $\beta_{ii^{\prime}}=-\frac{r_{3,i^\prime{}}-r_{2,i}}{d_{23}}$. For simplicity we denote $\alpha_{ii}$ and $\beta_{ii}$ by $\alpha_{i}$ and $\beta_{i}$, respectively. Given grouped echoes and Eqs. \eqref{rel1} and \eqref{d23-1}, we have
\begin{equation}\label{d23-2}
\theta_{i}=\pm \arccos \alpha_{i} \quad \text{and}\quad \theta_{i}-\varphi=\pm \arccos\beta_{i},
\end{equation}
There are four possible sign combinations for a given $i$,
\begin{equation}\label{EC-1}
\theta_{i}=\arccos \alpha_{i} \quad \text{and}\quad \theta_{i}-\varphi=\arccos\beta_{i}
\end{equation}
\begin{equation}\label{EC-2}
\theta_{i}=\arccos \alpha_{i} \quad \text{and}\quad \theta_{i}-\varphi=-\arccos\beta_{i}
\end{equation}
\begin{equation}\label{EC-3}
\theta_{i}=-\arccos \alpha_{i} \quad \text{and}\quad \theta_{i}-\varphi=\arccos\beta_{i}
\end{equation}
\begin{equation}\label{EC-4}
\theta_{i}=-\arccos \alpha_{i} \quad \text{and}\quad \theta_{i}-\varphi=-\arccos\beta_{i}.
\end{equation}
\begin{lemma}\label{L1}
Suppose $O_{j}\:(j=1,2,3)$ are feasible and not collinear. Given grouped first order echoes, with probability $1$, there exist exactly two sign combinations such that \eqref{rel1} and \eqref{d23-1} hold simultaneously for all $i$ if $\varphi$ and the direction of both $\overrightarrow{O_{1}O_{2}}$ and $\overrightarrow{O_{2}O_{3}}$ are randomly chosen. The two possible sign combinations have opposite signs for $\varphi$ and all $\theta_{i}$'s and correspond to reflection of each other in terms of recovered room shapes.
\end{lemma}
\begin{proof}
Assume without loss of generality that the ground truth of the polygon is \eqref{EC-1} for all $i\in\{1,\ldots,K\}$. Note that \eqref{EC-1} implies that \eqref{EC-4} holds for $\theta_{i}^{\prime}=-\theta_{i}$ and $\varphi^{\prime}=-\varphi <0$ for all $i$, i.e., they correspond to reflections of each other.\par
Suppose multiple sign combinations hold for a wall. Without loss of generality, let $i=1$. From \eqref{EC-1} we have
\begin{equation}\label{d23-pf1}
\varphi=\arccos\alpha_{1}-\arccos\beta_{1}.
\end{equation}
Assume that one of the following equations also holds,
\begin{equation}\label{d23-pf2}
\varphi=-\arccos\alpha_{1}-\arccos\beta_{1},
\end{equation}
\begin{equation}\label{d23-pf3}
\varphi=\arccos\alpha_{1}+\arccos\beta_{1},
\end{equation}
\begin{equation}\label{d23-pf4}
\varphi=-\arccos\alpha_{1}+\arccos\beta_{1}.
\end{equation}
Then we have the following three cases
\begin{enumerate}
    \item If \eqref{d23-pf1} and \eqref{d23-pf2} hold, we must have $\theta_{1}=0$ which implies that $O_{1}O_{2}$ is perpendicular to the first wall, and $\varphi=-\arccos\beta_{1}$.
    \item If \eqref{d23-pf1} and \eqref{d23-pf3} hold, we must have $\arccos\beta_{1}=0$, which implies that $O_{2}O_{3}$ is perpendicular to the first wall.
    \item If \eqref{d23-pf1} and \eqref{d23-pf4} hold, we must have $\varphi=0$, which contradicts the assumption that $O_{1}$, $O_{2}$ and $O_{3}$ are not collinear.
\end{enumerate}
Given that the three measurement points are randomly chosen, and, subsequently, $\varphi$, $\overrightarrow{O_{1}O_{2}}$ and $\overrightarrow{O_{2}O_{3}}$ are random, the first two cases do not occur with probability one.\par
If a subset of \eqref{EC-2}-\eqref{EC-4} holds for $i$ and $i^{\prime}$ simultaneously, then we must have $(\theta_{i},\theta_{i^{\prime}})\in\{\theta_{i}=0,\theta_{i}=\varphi,\varphi=0\}\times\{\theta_{i^{\prime}}=0,\theta_{i^{\prime}}=\varphi,\varphi=0\}$, which again, do not occur due to randomly chosen measurement points. Similarly, it can be shown that for more than $2$ walls, \eqref{EC-1} would imply none of \eqref{EC-2}-\eqref{EC-4} holds for all walls.
\end{proof}
\subsection{Echo Labeling}
Since echoes may arrive in different orders at different $O_{j}$'s and $\tilde{\mathbf{r}}_{j}$ contains higher order echoes if $N_{j}>K$, $f_{j}$ is usually unknown. We say the received echoes are \emph{ungrouped} if $f_{j}$ is unknown for some $j$. Thus given $\tilde{\mathbf{r}}_j$, our task is to first determine the mapping $f_{j}$, i.e., label the echoes, followed by estimation of $\theta_{i}$'s and $\varphi$. \par
\begin{lemma}\label{L2}
With ungrouped echoes, any mapping $f_{j}^{\prime}$ that differs from the correct mapping $f_{j}$ will result, with probability $1$, the following two possible cases
\begin{enumerate}
\item there exists no solution to \eqref{rel1} and \eqref{d23-1} given no parallel edges, or
\item the reconstructed room shape has larger dimension with respect to parallel edges.
\end{enumerate}
\end{lemma}
\begin{proof}
We illustrate the proof by considering the case $K=4$. The result can be easily extended to $K=3$ and $K>4$.\par
Suppose again that the ground truth is \eqref{EC-1} for all $i$. We first consider parallelograms and exclude odd higher order echoes resulting from a pair of parallel walls. The distances between $O_{j}\:(j=1,2,3)$ and the four walls satisfy
\begin{equation}\label{mic-pf-p4}
r_{1,1}+r_{1,2}=r_{2,1}+r_{2,2}=r_{3,1}+r_{3,2}=a,
\end{equation}
\begin{equation}\label{mic-pf-p5}
r_{1,3}+r_{1,4}=r_{2,3}+r_{2,4}=r_{3,3}+r_{3,4}=b.
\end{equation}
One can see that for some $f_{j}^{\prime}$'s, pairs of $\{\alpha_{ii^{\prime}},\beta_{ii^{\prime}}\}\: (i,i^{\prime}\in\{1,2,3,4\})$ are related to each other. Consider for example the $f_{j}^{\prime}$'s resulting in $\{\alpha_{12},\alpha_{21}, \alpha_{34}, \alpha_{43}\}$ and $\{\beta_{12},\beta_{21}, \beta_{34}, \beta_{43}\}$. Since $\alpha_{12}+\alpha_{21}=0$, $\alpha_{34}+\alpha_{43}=0$, $\beta_{12}+\beta_{21}=0$ and $\beta_{34}+\beta_{43}=0$, we have
\begin{equation*}
\arccos(\alpha_{21})=\pi\pm\arccos(\alpha_{12}),
\end{equation*}
\begin{equation*}
\arccos(\alpha_{43})=\pi\pm\arccos(\alpha_{34}),
\end{equation*}
\begin{equation*}
\arccos(\beta_{21})=\pi\pm\arccos(\beta_{12}),
\end{equation*}
\begin{equation*}
\arccos(\beta_{43})=\pi\pm\arccos(\beta_{34}).
\end{equation*}
Thus \eqref{d23-2} reduces to two equations
\begin{equation*}
\varphi=\pm \arccos(\alpha_{12})\pm\arccos(\beta_{12}),
\end{equation*}
\begin{equation*}
\varphi=\pm \arccos(\alpha_{34})\pm\arccos(\beta_{34}).
\end{equation*}
With probability $1$, these two equations do not hold simultaneously as $\alpha_{12}$, $\beta_{12}$ are independent of $\alpha_{34}$, $\beta_{34}$ due to randomly chosen measurement points. Other $f_{j}^{\prime}(\neq f_{j})$'s always have at least two equations with independent choice of $\alpha$ and $\beta$. Hence no solution can be found for those instances.\par
Suppose $f_{j}^{\prime}$'s are chosen such that we have $\alpha_{ii^{\prime}}$ and $\beta_{ii^{\prime\prime}}$ ($i\neq i^{\prime}, \: i\neq i^{\prime\prime}$). For rooms with no more than one pair of parallel walls, only echoes chosen according to $f_{j}$'s satisfy \eqref{EC-1} for all $i$. This is because for those rooms, at least one of \eqref{mic-pf-p4} and \eqref{mic-pf-p5} does not hold. Thus some $\alpha_{ii^{\prime}}$'s and $\beta_{ii^{\prime\prime}}$'s are not related since $r_{1i^{\prime}}$, $r_{2i}$ and $r_{3i^{\prime\prime}}$ are randomly chosen from $\tilde{\mathbf{r}}_{1}$, $\tilde{\mathbf{r}}_{2}$ and $\tilde{\mathbf{r}}_{3}$, respectively.\par
Given parallel edges, however, higher order echoes may also satisfy \eqref{rel1} and \eqref{d23-1}. For instance, as shown in Fig.~$3$, suppose that walls $1$ and $3$ are parallel. Then it is easy to verify that
\begin{equation*}
r_{j,131} - r_{j^{\prime},131} = r_{j,1} - r_{j^{\prime},1},
\end{equation*}
\begin{equation*}
r_{j,313} - r_{j^{\prime},313} = r_{j,3} - r_{j^{\prime},3},
\end{equation*}
where $j\neq j^{\prime}$. Hence, \eqref{rel1} and \eqref{d23-1} provide the same $\cos\theta_{1}$, $\cos\theta_{3}$, $\cos(\theta_{1}-\varphi)$ and $\cos(\theta_{3}-\varphi)$ if $r_{j,1}$ and $r_{j,3}$ are replaced by $r_{j,131}$ and $r_{j,313}$, respectively. By Lemma III.1, the third order echoes resulting from a pair of parallel edges lead to a larger room with the same norm vectors. Exactly the same argument applies to odd higher order echoes from a pair of parallel edges. Therefore, Lemma III.2 is proved.
\end{proof}
\emph{Remark} 1: The ambiguities resulting from parallel edges can be easily eliminated if we always choose SLAM result with the smallest room size.\par
Given Lemma III.1 and Lemma III.2, we have the following result on the identifiability of any convex polygonal room by using only first order echoes.
\begin{theorem}\label{L3}
With probability $1$, SLAM can be achieved subject to reflection ambiguity given any convex planar $K$-polygon, by using the first order echoes received at three random points in the feasible region, with known $d_{12}$ and $d_{23}$ and unknown $\varphi\in(0,2\pi)$.
\end{theorem}
\begin{figure}[btp]
    \center
    \includegraphics[height=45mm,width=75mm]{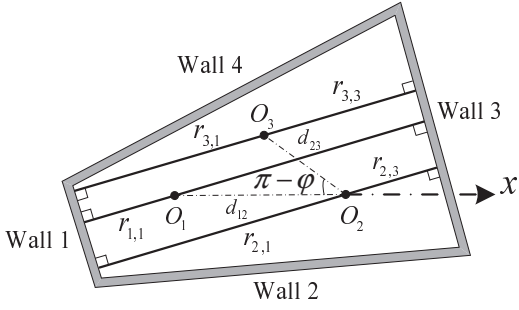}
    \caption{A room with a pair of parallel edges. Here wall $1$ and $3$ are parallel.}
    \label{Fig3}
\end{figure}
\emph{Remark} 2: Both the room shape and the coordinate of $O_{3}$ are subject to reflection ambiguity for $\varphi\in(0,2\pi)$. If, however, we can limit $\varphi\in(0,\pi)$, SLAM will be free of such ambiguity.\par
\emph{Remark} 3: In reality, it is inevitable to collect reflections from the ceiling and the floor. However, by Theorem III.3, if distances corresponding to these echoes are included, no polygon can be recovered provided that the trajectory of the device lies in a plane that is perpendicular to the walls.\par
\subsection{A Practical Algorithm}
In a real acoustic system, $m^{(j)}(t)$'s in \eqref{correlation1} are inevitably corrupted by measurement noise leading to corrupted measurement of $\tilde{\mathbf{r}}_{j}$. Let the corrupted version of $\tilde{\mathbf{r}}_{j}$ be denoted by $\hat{\mathbf{r}}_{j}$. Two issues arise. First, given $f_{j}$'s, $\varphi$ obtained by \eqref{d23-2} for different $i$'s are not necessarily identical. The second issue is the possibility that the computed cosine values in \eqref{rel1} may have absolute value exceeding $1$. For the former, we propose a heuristic scheme of choosing the echo and sign combination that yield the smallest variance of the estimated $\varphi$'s across different $i$'s. Notice that in the noiseless case with perfect echo measurements, the variance of the estimated $\varphi$'s across different $i$'s is $0$ if the correct echo and sign combination is selected while all others will have non-zero (potentially large variance). For the latter, define a feasible $\cos\theta_{i}$ as
\begin{equation*}
\cos\theta_{i} = \begin{cases}
1, & \text{if}\ 1\leq -\frac{\hat{r}_{2,i}-\hat{r}_{1,i}}{d_{12}} < 1+\epsilon \\
-\frac{\hat{r}_{2,i}-\hat{r}_{1,i}}{d_{12}}, & \text{if}\ -1<-\frac{\hat{r}_{2,i}-\hat{r}_{1,i}}{d_{12}}<1 \\
-1, & \text{if}\ -1-\epsilon < -\frac{\hat{r}_{2,i}-\hat{r}_{1,i}}{d_{12}} \leq -1
\end{cases},
\end{equation*}
where $\epsilon > 0$ is a tuning parameter determined by the noise level. Feasible $\cos(\theta_{i}-\varphi)$ can be similarly defined. The echo combination is said to be infeasible if either $|\frac{\hat{r}_{2,i}-\hat{r}_{1,i}}{d_{12}}|>1+\epsilon$ or $|\frac{\hat{r}_{3,i}-\hat{r}_{2,i}}{d_{23}}|>1+\epsilon$. Only those feasible $\theta_{i}$'s and $\varphi$ will be used in computing the variance of the estimated $\varphi$.\par
As the number of walls for the room is not known in prior, the proposed algorithm needs to first reconstruct some room shapes with $K=3,\ldots,N$ walls. Then the desired room shape is the feasible one with the largest number of walls. In order to reconstruct a room shape with $K$ walls, the number of echo combinations that need to be exhausted is
\begin{equation*}
\binom{N_{1}}{K}\binom{N_{2}}{K}\binom{N_{3}}{K}(K!)^2.
\end{equation*}
For simplicity assume that $N=N_{1}=N_{2}=N_{3}$. Let $V_{th}$ be the threshold of the variance. The corresponding algorithm is summarized as Algorithm 1.
\begin{algorithm}
\caption{Reconstruct convex polygon given distances between consecutive measurement points}
\label{CP-recon}
\begin{algorithmic}[1]
\State {Set $K=3$ and $V_{th}$.}
\If {$K\leq N$}
\State {Set $V_{K}=\inf$ and the stored polygon with $K$ walls be empty.}
\For {$n=1:\big(\binom{N}{K}\big)^{3}(K!)^2$}
\State {Based on the $n$th echo combination, choose $K$ elements from $\hat{\mathbf{r}}_{1}$, $\hat{\mathbf{r}}_{2}$, $\hat{\mathbf{r}}_{3}$, respectively.}
\State {Compute $\cos\theta_{i}$'s and $\cos(\theta_{i}-\varphi)$ for $i=1,\ldots,K$.}
\If {$\cos\theta_{i}$'s and $\cos(\theta_{i}-\varphi)$ are feasible}
\State {Compute $\text{Var}[\varphi]$ for different sign combinations and keep the one with the smallest $\text{Var}[\varphi]$.}
\If {$\text{Var}[\varphi]< V_{K}$ and the room shape does not fully cover the stored one with $K$ walls}
\State {Keep the echo and sign combination and set $V_{K} = \text{Var}[\varphi]$ for $K$.}
\EndIf
\EndIf
\EndFor
\State {$K=K+1$.}
\Else
\State{Keep the SLAM results the largest $K$ such that $V_{K} < V_{th}$.}
\EndIf
\end{algorithmic}
\end{algorithm}
\subsection{SLAM with One Path Length}
Now that we have established that two distances between three consecutive measurement points are sufficient to overcome the drawback of using first order echoes alone, a natural question is what would be the least amount of information that is required to achieve SLAM for any convex polygons. Specifically we examine the case where only one distance between a pair of measurement points is known. We show that for a parallelogram, there exist multiple rooms satisfying \eqref{rel1} and \eqref{d12-1} in this case, thus the answer is negative, i.e. a single distance measurement is insufficient for SLAM with ungrouped first order echoes.\par
Without loss of generality, assume $d_{12}$ is known but $d_{23}$ is not. As shown in Fig.~$4$, let $O_{1}$ be the origin, $O_{2}$ be on the x-axis and $O_{3}(x_{3},y_{3})\:(y_{3}\neq 0)$ is unknown. We also assume that the direction of $\overrightarrow{O_{1}O_{2}}$ with respect to the desired room is unknown. By geometry, we have \eqref{rel1} and
\begin{equation}\label{d12-1}
(r_{3,i}-r_{1,i})+x_{3}\cos\theta_{i}+y_{3}\sin\theta_{i}=0.
\end{equation}
Eq. \eqref{d12-1} can also be rewritten in a matrix form
\begin{equation}\label{Matr1}
\mathbf{A}[x_{3},y_{3}]^{T}=\mathbf{b},
\end{equation}
where
\begin{equation*}
\mathbf{A} =\begin{bmatrix}
\cos\theta_{1} & \sin\theta_{1} \\
\vdots & \vdots \\
\cos\theta_{K} & \sin\theta_{K}
\end{bmatrix},
\end{equation*}
and
\begin{equation*}
\mathbf{b}=[-(r_{3,1}-r_{1,1}),\ldots,-(r_{3,K}-r_{1,K})]^{T}.
\end{equation*}
\begin{figure}[btp]
    \center
    \includegraphics[height=48mm,width=72mm]{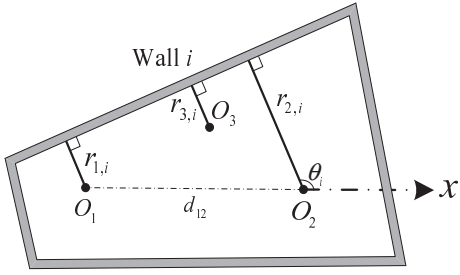}
    \caption{A mobile device is employed to measure the geometry of a room. The mobile device collects echoes at $O_{1}$, $O_{2}$ and $O_{3}$ successively. Only the distances between $O_{1}$ and $O_{2}$ ($d_{12}$) is known.}
    \label{Fig3}
\end{figure}

The ground truth of a parallelogram is assumed to be
\begin{equation*}
\mathbf{A}=\begin{bmatrix}
\cos\theta_{1} & \sin\theta_{1} \\
\cos\theta_{2} & \sin\theta_{2} \\
\cos\theta_{3} & \sin\theta_{3} \\
\cos\theta_{4} & \sin\theta_{4}
\end{bmatrix}\:\text{and}\:\;
\mathbf{b}=\begin{bmatrix}
-(r_{3,1}-r_{1,1}) \\
-(r_{3,2}-r_{1,2}) \\
-(r_{3,3}-r_{1,3}) \\
-(r_{3,4}-r_{1,4})
\end{bmatrix},
\end{equation*}
where
\begin{equation*}
r_{1,1}+r_{1,3}=r_{2,1}+r_{2,3}=r_{3,1}+r_{3,3},
\end{equation*}
\begin{equation*}
r_{1,2}+r_{1,4}=r_{2,2}+r_{2,4}=r_{3,2}+r_{3,4}.
\end{equation*}
Let
\begin{equation*}
\mathbf{A}^{\prime}=\begin{bmatrix}
\cos\theta_{13} & \sin\theta_{13} \\
\cos\theta_{24} & \sin\theta_{24} \\
\cos\theta_{31} & \sin\theta_{31} \\
\cos\theta_{42} & \sin\theta_{42}
\end{bmatrix}
\mathbf{b}^{\prime}=\begin{bmatrix}
-(r_{3,1}-r_{1,3}) \\
-(r_{3,2}-r_{1,4}) \\
-(r_{3,3}-r_{1,1}) \\
-(r_{3,4}-r_{1,2})
\end{bmatrix}.
\end{equation*}
Then
\begin{equation*}
\cos\theta_{13}+\cos\theta_{31}=0\quad\text{and}\quad\cos\theta_{24}+\cos\theta_{42}=0.
\end{equation*}
Moreover, since $\sin\theta=\pm\sqrt{1-\cos^{2}\theta}$,
\begin{equation*}
\sin\theta_{13}+\sin\theta_{31}=0\quad\text{and}\quad\sin\theta_{24}+\sin\theta_{42}=0
\end{equation*}
can hold if we manipulate the sign of square root properly.\par
Then $\text{rank}(\mathbf{A}^{\prime})=\text{rank}([\mathbf{A}^{\prime},\mathbf{b}^{\prime}])=2$. Thus a room shape and the coordinate of $O_{3}$ different from the ground truth and its reflection also satisfy both \eqref{rel1} and \eqref{d12-1}.\par

\section{Experimental results}
\subsection{Experiment Setup}
We describe in the following some preliminary experimental results. Enormous challenges exist to conduct a truly autonomous SLAM. Chief among them are: the search space (number of combinations) is extremely large - using for example, some modest numbers, e.g. $K=4$ and $N_{1}=N_{2}=N_{3}=8$, the number of echo combinations exceeds $10^{7}$, combining with the sign combinations the search space is in the billions; the measurement of motion sensors is still subject to large errors and some robustification of the reconstruction algorithm will need to be investigated if the true motion sensor measurements are used. The purpose of the experimental design is thus to demonstrate the feasibility of the proposed scheme in an idealized situation with a certain degree of human intervention to alleviate the above challenges.\par
We use a laptop as a microphone and a HTC M8 phone as our loudspeaker. As the loudspeaker of the cell phone is not omnidirectional and is power limited, we place the speaker of the cell phone towards each wall to ensure the corresponding first order echo is strong enough. Note that the microphone will record both first order echoes and some higher order ones. A chirp signal linearly sweeping from $30$Hz to $8$kHz is emitted by the cell phone. The sample rate at the receiver is $f_{s}=96$kHz. It has been shown in \cite{Farina,Stan} that if the input chirp signal is correlated with its windowed version, the output may resemble a delta function, which is desirable for better delay resolution.
\begin{figure}[!ht]
    \subfloat[Transmitted signal convolves with itself\label{subfig-4:dummy}]{%
      \includegraphics[width=0.23\textwidth]{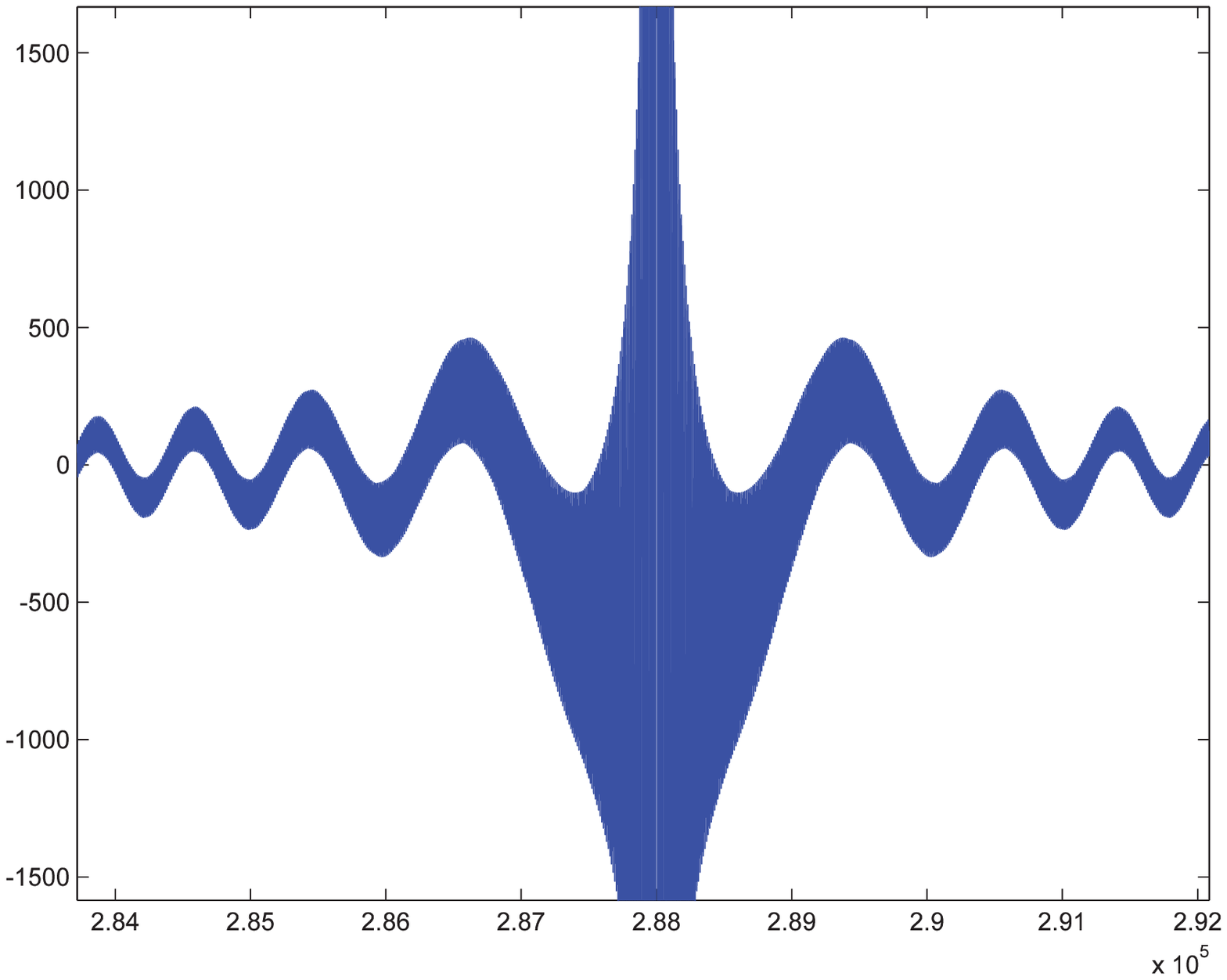}
    }
    \hfill
    \subfloat[Transmitted signal convolves with its windowed version\label{subfig-5:dummy}]{%
      \includegraphics[width=0.23\textwidth]{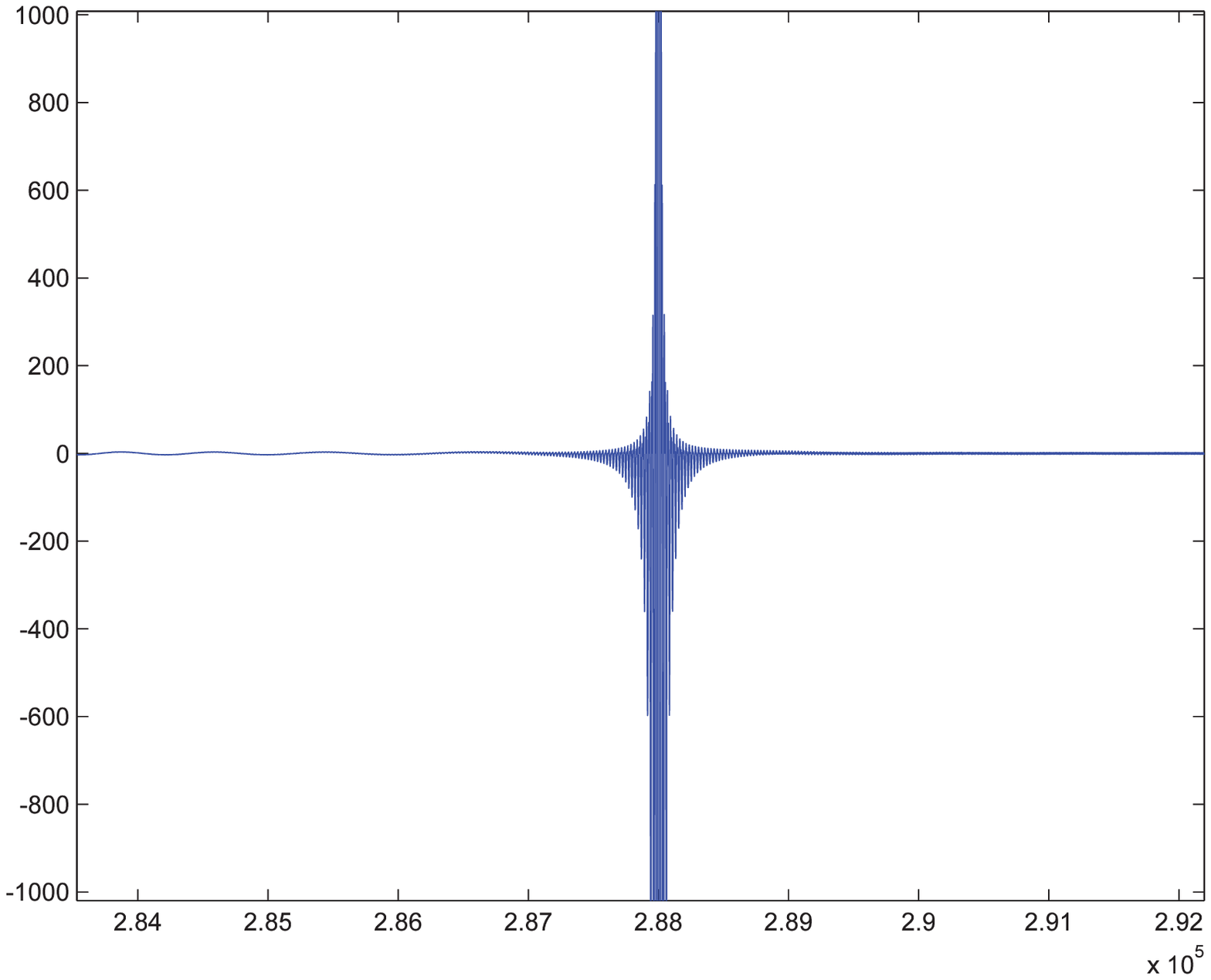}
    }
    \caption{Comparison of convolution result. The maximum values of the two convolution result are set to be identical.}
    \label{fig:dummy}
  \end{figure}
Our simulation indicates that correlating the received signals with its triangularly windowed version outperforms the correlator using the original one. The comparison is shown in Fig.~$5$.\par
\begin{figure}[!ht]
    \subfloat[Correlator output at $O_{1}$ towards the first wall\label{subfig-1:dummy}]{%
      \includegraphics[width=0.4\textwidth]{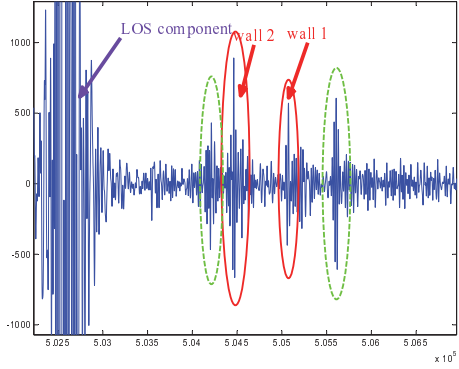}
    }
    \hfill
    \subfloat[Correlator output at $O_{2}$ towards the second wall\label{subfig-2:dummy}]{%
      \includegraphics[width=0.4\textwidth]{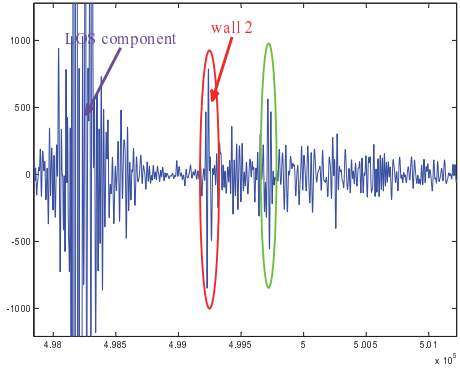}
    }
    \hfill
    \subfloat[Correlator output at $O_{3}$ towards the third wall\label{subfig-3:dummy}]{%
      \includegraphics[width=0.4\textwidth]{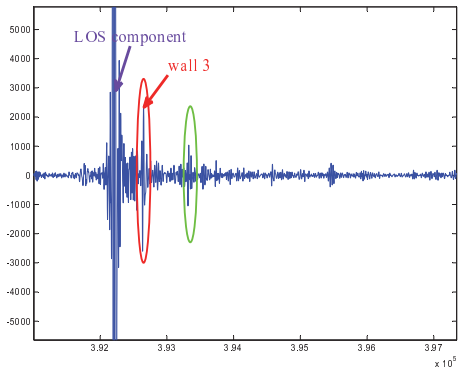}
    }
    \caption{Sample of correlator output: Peaks with solid red ellipses correspond to walls while peaks with dash green ellipses correspond to either noise or higher order echoes}
    \label{fig:dummy}
\end{figure}
Fig.~$6$ is a sample path of the correlator output collected in the room where this experiment is conducted. In Fig.~$6$, peaks marked with red ellipse are desired while those with green ellipse correspond to noise, the ceiling, the floor, higher order echoes or other spurious sources. In our experiment, we use $|m^{(j)}(t)|$ rather than $m^{(j)}(t)$ since the true peaks may be either positive or negative. Local maxima of $|m^{(j)}(t)|$ corresponding to Fig.~$6$ are shown in Fig.~$7$.\par
A heuristic way to detect peaks, summarized in Algorithm 3, is to check the slope of each local maxima. Three requirements are needed for the proposed algorithm: 1) the minimum distance between the device and the walls is no less than $d_{min}$, 2) the minimum TDOA of two detected consecutive echoes is no less than $\Delta t$, 3) the maximum candidate distance corresponding to detected peaks is no more than $d_{max}$. The reason for the requirements is as follows: 1) since the correlation property of the chirp signal is not ideal and the power of the LOS component is much larger than that of reflective components, the distance between the device and the walls should be large enough such that the peaks corresponding to reflective components are not overshadowed by the LOS component, 2) as the power of reflective paths decays rapidly, it is reasonable to restrict the detectable echoes within certain distances which depends on the power of loudspeaker. Given $d_{min} = 0.6$m, $d_{max}=6.5$m and $\Delta t = \frac{0.5\text{m}}{c}$, where $c=346\text{m/s}$, the detection results are marked by arrows in Fig.~$7$. We can see that the desired peaks are always detected. In order to detect as less false peaks as possible, one possible modification is to apply a tapering threshold which decreases as $t$ increases.\par
\begin{algorithm}
\caption{Peak detection algorithm}
\label{Peak}
\begin{algorithmic}[1]
\State find LOS peak $(t_{0}^{(j)},m_{0}^{(j)})$.
\State find local maxima of $|m^{(j)}(t)|$ appearing from $t_{0}^{(j)}+t_{min}$ to $t_{0}^{(j)}+t_{max}$.
\State find all peaks that are \emph{peaky} and store them in $M$
\State set $P = {\O}$
\If $|P| < |M|$
\If {there exist peaks in $M$ whose locations are "close" to any peak in $P$}
\State remove those peaks from $M$.
\Else
\State add the peak with the largest magnitude of $M$ to $P$.
\EndIf
\EndIf
\end{algorithmic}
\end{algorithm}
\begin{figure}[!ht]
    \subfloat[Peaks detected from correlator output at $O_{1}$ towards the first wall\label{subfig-6:dummy}]{%
      \includegraphics[width=0.45\textwidth]{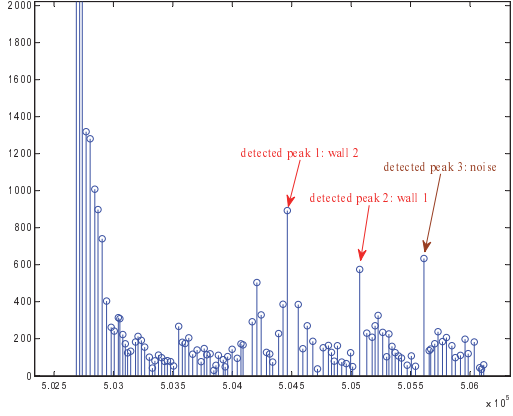}
    }
    \hfill
    \subfloat[Peaks detected from correlator output at $O_{2}$ towards the second wall\label{subfig-7:dummy}]{%
      \includegraphics[width=0.45\textwidth]{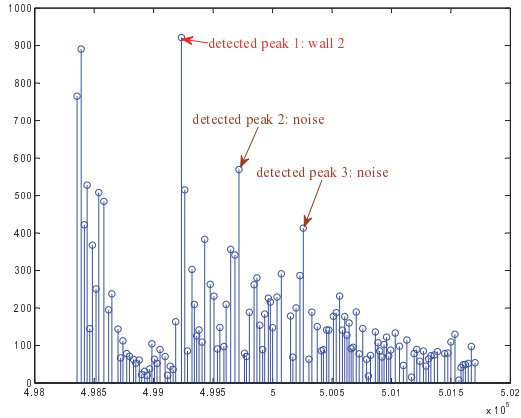}
    }
    \hfill
    \subfloat[Peaks detected from correlator output at $O_{3}$ towards the third wall\label{subfig-8:dummy}]{%
      \includegraphics[width=0.45\textwidth]{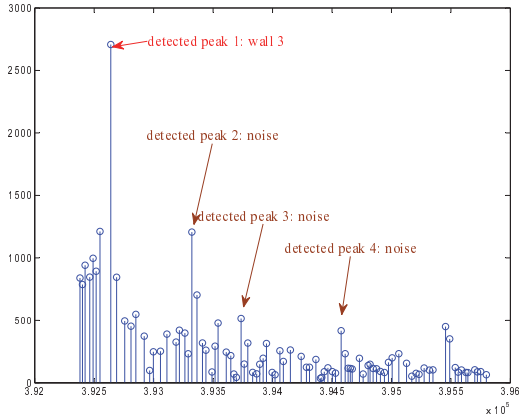}
    }
    \caption{Illustration of the performance of the proposed peak detection algorithm.}
    \label{fig:dummy}
  \end{figure}
\subsection{SLAM Results}
Echoes are collected at $O_{j}\:(j=1,\ldots,4)$ and $d_{j,j+1}\:(j=1,2,3)$ are measured by tape measure. The proposed peak detection algorithm is used to estimate the candidate distances from received signals. Note that the number of detected peaks are much larger than the number of first order echoes. Heuristics are used to remove peaks (e.g. those of small magnitudes) - otherwise, checking all combinations of echoes become computationally prohibitive. The proposed algorithm for SLAM is verified by experiment at $O_{1}$, $O_{2}$, $O_{3}$ and $O_{2}$, $O_{3}$, $O_{4}$. Given $O_{2}$, $O_{3}$, $O_{4}$, we assume that $O_{2}$ is the origin and $O_{3}$ lies on the $x$-axis. Even if some elements of $\mathbf{r}_{j}$ have measurement errors up to $10$cm, SLAM is accomplished with small error of both the room shape and the coordinates of $O_{3}$ and $O_{4}$ with only unlabeled first-order echoes. In the presence of higher order echoes, the proposed algorithm may perform poorly and ambiguity may occur when the variance of $\varphi$ is the only criterion used to determine $f_{j}$'s. With noisy measurement, it is possible that the incorrect echo combination may yield feasible $\theta_{i}$ and $\varphi$ with variance smaller than that of the correct echo combination. Furthermore, an interesting phenomenon is that sometimes the proposed algorithm is unable to provide the correct room shape, but the estimate of $\varphi$ is always close to the true value. This means that better echo labeling approach is needed for robust SLAM. 
As most rooms are regular, we add a heuristic constraint: all the angles of two adjacent walls are between $50^{\circ}$ and $130^{\circ}$. The comparison between the SLAM result and the ground truth is illustrated in Fig.~$8$. The candidate distances are obtained by the peak detection algorithm. Note that the coordinate system in Fig.~$8$(b) is a rotation of that in Fig.~$8$(a) by $135^{\circ}$ counterclockwise. The SLAM results shown in the two figures are rotational images of each other. Experimental result indicate that heuristic constraints such as the above can largely eliminate incorrect combinations.\par
\begin{figure}[!ht]
    \subfloat[SLAM via echoes collected at $O_{1}$, $O_{2}$ and $O_{3}$ \label{subfig-4:dummy}]{%
      \includegraphics[width=0.45\textwidth]{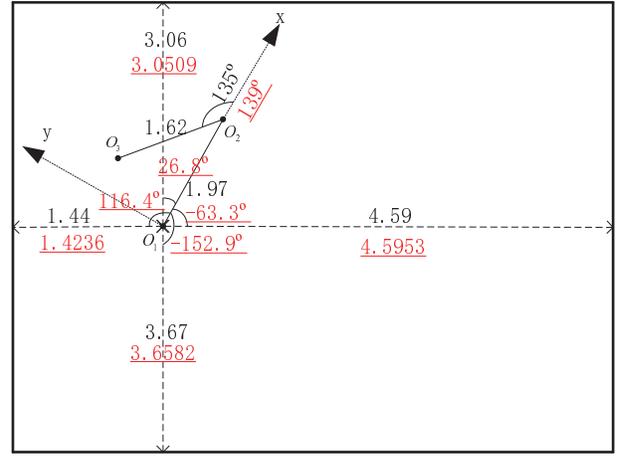}
    }
    \hfill
    \subfloat[SLAM via echoes collected at $O_{2}$, $O_{3}$ and $O_{4}$ \label{subfig-5:dummy}]{%
      \includegraphics[width=0.45\textwidth]{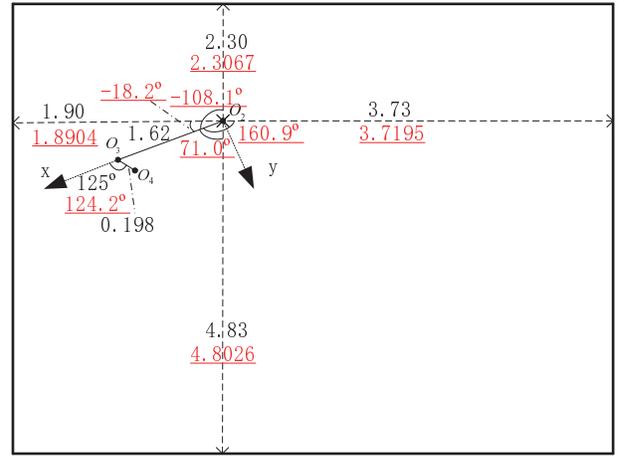}
    }
    \caption{Comparison between the ground truth (black) and experiment result (red underlined)}
    \label{fig:dummy}
  \end{figure}

\section{Conclusion}
This work makes progress in acoustic SLAM using a single mobile device with unlabeled first order echoes. Theoretical guarantee of $2$-D SLAM is established when two path lengths corresponding to three consecutive measurement points are available. Conversely, it was also shown that with only a single distance measurement, $2$-d SLAM with unlabeled first order echoes is not possible for all convex polygons. The result is summarized in Table I.\par
\begin{table}[h]
\centering
\caption{Feasibility of SLAM with unlabeled first order echoes and different geometry knowledge}
\begin{tabular}{ | c | c | p{'1.0cm'} |}
\hline
geometry knowledge	& any convex polygon \\ \hline
$d_{12}$, $d_{23}$, $d_{13}$ & Yes \\  \hline
$d_{12}$, $d_{23}$ & Yes  \\ \hline
$d_{12}$ & No  \\ \hline
none & No \\ \hline
\end{tabular}
\end{table}

While theoretical guarantee can be established for the noiseless case, the proposed algorithm needs to be enhanced to ensure a fully autonomous $2$-D SLAM. Two particular issues that need to be further addressed include the robustness with respect to measurement noise and the computational complexity when a large number of peaks are detected at each measurement location.
\label{sec:refs}

\bibliographystyle{IEEEbib}
\bibliography{Echo}

\end{document}